%% file: arxiv.tex
\documentclass[11pt, letterpaper]{article}
\pdfoutput=1 
\usepackage{fullpage}
\usepackage[numbers]{natbib}
\usepackage{microtype}
\usepackage[width=374.18663pt]{caption}
\usepackage{graphicx}
\usepackage{xcolor}
\usepackage{subfigure}
\usepackage{booktabs} %

\definecolor{Darkblue}{rgb}{0,0,0.4}
\definecolor{Brown}{cmyk}{0,0.81,1.,0.60}
\definecolor{Purple}{cmyk}{0.45,0.86,0,0}

\usepackage{hyperref}

\hypersetup{colorlinks=true,%
           citebordercolor={.6 .6 .6},linkbordercolor={.6 .6 .6},%
citecolor=Darkblue,urlcolor=black,linkcolor=red}
\newcommand{\lref}[2][]{\hyperref[#2]{#1~\ref*{#2}}}

\usepackage{amsmath}
\usepackage{amssymb}
\usepackage{mathtools}
\usepackage{amsthm}
\usepackage{thm-restate}

\usepackage{enumitem}
\usepackage[capitalize,noabbrev]{cleveref}
\usepackage{wrapfig}

\theoremstyle{plain}
\newtheorem{theorem}{Theorem}[section]

\theoremstyle{definition}

\newtheorem{assumption}[theorem]{Assumption}
\theoremstyle{remark}

\newtheorem{claim}[theorem]{Claim}
\newtheorem{fact}[theorem]{Fact}

\usepackage[textsize=tiny]{todonotes}

\input{macros}

\title{Stacking as Accelerated Gradient Descent}

\let\citet\cite

\author{
Naman Agarwal \\
Google DeepMind\\
\texttt{namanagarwal@google.com}\\
\and
Pranjal Awasthi \\
Google Research\\
\texttt{pranjalawasthi@google.com}
\and
Satyen Kale \\
Google Research\\
\texttt{satyenkale@google.com}
\and
Eric Zhao\\
UC Berkeley, Google Research\\
\texttt{eric.zh@berkeley.edu} 
\and
}
\begin{document}
\def\arxiv{1}

\date{} 

\maketitle

\begin{abstract}
    Stacking, a heuristic technique for training deep residual networks by progressively increasing the number of layers and initializing new layers by copying parameters from older layers, has proven quite successful in improving the efficiency of training deep neural networks. In this paper, we propose a theoretical explanation for the efficacy of stacking: viz., stacking implements a form of Nesterov's accelerated gradient descent. The theory also covers simpler models such as the additive ensembles constructed in boosting methods, and provides an explanation for a similar widely-used practical heuristic for initializing the new classifier in each round of boosting. We also prove that for certain deep linear residual networks, stacking does provide accelerated training, via a new potential function analysis of the Nesterov's accelerated gradient method which allows errors in updates. We conduct proof-of-concept experiments to validate our theory as well.
\end{abstract}

\section{Introduction}
\label{sec:intro}
\input{src/introduction}

\section{Stagewise training as functional gradient descent}
\label{sec:functional}
\input{src/fgd}

\section{Accelerated convergence of deep linear networks by stacking}
\label{sec:deep-linear-theory}
\input{src/main_results}

\section{Experiments}
\label{sec:experiments_main}
\input{src/experiments}

\section{Conclusions and Future Work}
This paper develops the theoretical perspective that the effectiveness of stacking initialization, compared to other forms of initialization such as zero or random, is because it enables a form of accelerated gradient descent in function space. There are several directions for future work. While this work provides a formal proof of accelerated convergence for a particular parametric setting (deep residual linear networks), such a proof in the general functional setting for deep residual networks is still open, and will probably require some additional assumptions. From a practical standpoint, a very intriguing and potentially impactful question is whether it is possible to come up with an efficiently implementable initialization scheme that leads to Nesterov's AGD updates \textit{exactly} for deep residual networks. 

\bibliographystyle{alpha}
\input{arxiv.bbl}

\end{document}

%% file: macros.tex
\newcommand{\lipschitz}{L}
\newcommand{\Tr}{\text{Tr}}

\renewcommand{\epsilon}{\varepsilon}
\renewcommand{\hat}{\widehat}
\renewcommand{\tilde}{\widetilde}

\newcommand{\inner}[2]{\left< #1, #2 \right>}

\newcommand{\norm}[1]{\left\lVert#1\right\rVert}

\newcommand{\para}[1]{\left(#1\right)}

\newcommand{\abs}[1]{\left|#1\right|}

\DeclareMathOperator*{\Exp}{\mathbb{E}}
\newcommand{\E}{\Exp}

\renewcommand{\cite}[1]{\citep{#1}}

\newcommand{\reals}{\mathbb{R}}

\newcommand{\asseq}{\coloneqq}

\newcommand{\loss}{\ell}

\newcommand{\cF}{\mathcal{F}}

\newcommand{\cY}{\mathcal{Y}}

%% file: src/introduction.tex
Deep learning architectures are ubiquitous today and have been responsible tremendous technological advances in machine learning. However, until 2006, training deep architectures was extremely challenging. The deep learning revolution of the past couple of decades was ushered in with the discovery that a classical technique, viz. \emph{greedy layer-wise pretraining}, can be used to train general deep architectures (see \citep[Section 15.1]{Goodfellow-et-al-2016} for a historical account of this). 
Previously, only deep architectures with special structure like convolutions or recurrences were known to be feasible to train. Greedy layer-wise pretraining is a very intuitive technique where a deep network is built and trained in a stagewise manner. Starting with a small network that is easy to train, the technique prescribes adding new layers over a number of stages, and in each stage training the newly added layers (and, potentially, the older layers as well) for a certain number of training steps. This process continues until the desired model depth is reached. 

More modern developments such as residual connections \cite{kaiminghe} and normalization layers \cite{batchnorm} have made it possible to directly train deep networks without using greedy layer-wise pretraining. However, in recent years, the tremendous success of deep learning architectures based on transformers \citep{vaswani2017attention} in domains such as language modeling, and computer vision \citep{radford2019language, brown2020language, chen2022pali, tu2023towards} has led to the trend of scaling model capacity with ever increasing model sizes for improved performance 
\citep{chowdhery2022palm, achiam2023gpt,team2023gemini}. This endeavour comes at a significant cost as model training may often take months and require several million dollars of compute resources \citep{chowdhery2022palm}. As a result there has been a surge of recent work aimed at faster training of large transformer based models. These works include methods for sparse training such as mixture of experts (MoE) models \citep{shazeer2017outrageously, jiang2024mixtral}, methods for approximate sparse attention mechanisms \citep{beltagy2020longformer} and better optimization methods \citep{shazeer2018adafactor,liu2023sophia,shampoo}.

\if\arxiv0
\begin{figure}[ht]
\centering
\begin{minipage}{.5\textwidth}
  \centering
  \includegraphics[width=.98\textwidth]{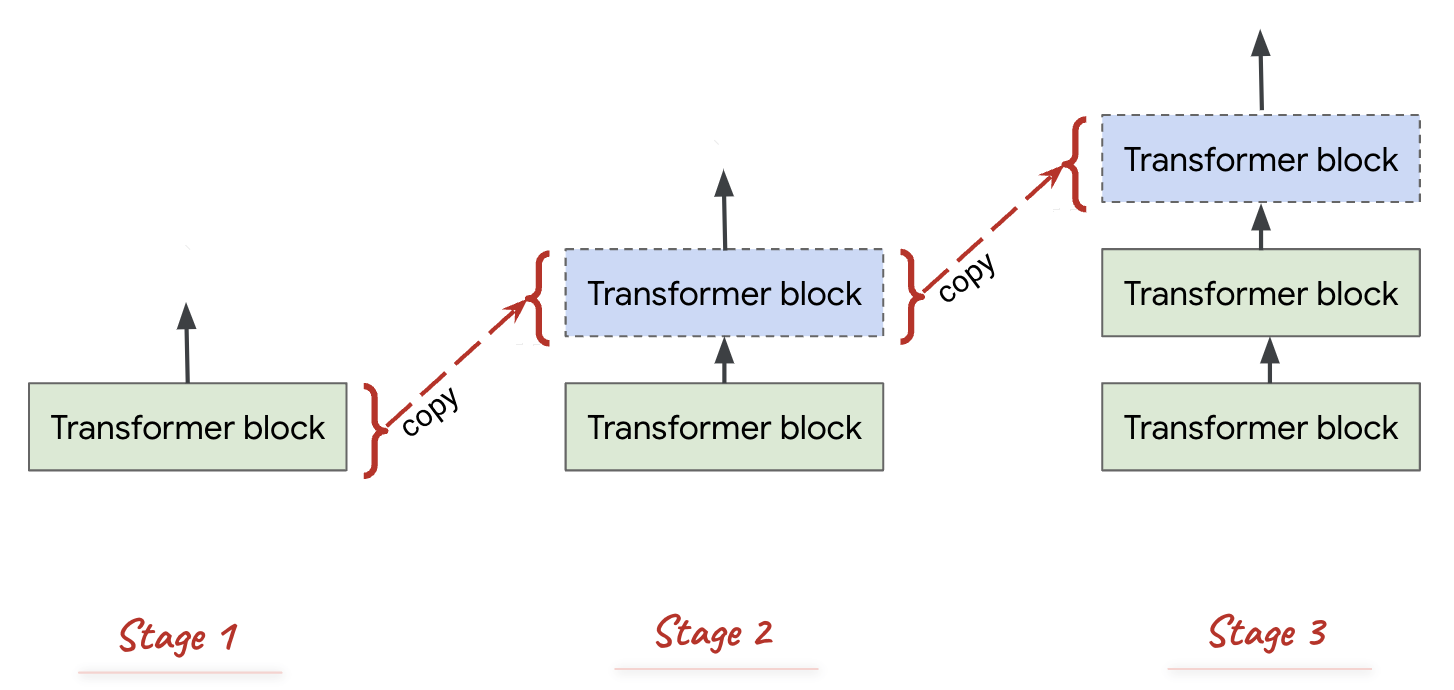}
  \caption{Stacking for language models.}
  \label{fig:stacking-lm}
\end{minipage}%
\begin{minipage}{.5\textwidth}
  \centering
  \includegraphics[width=.8\textwidth]{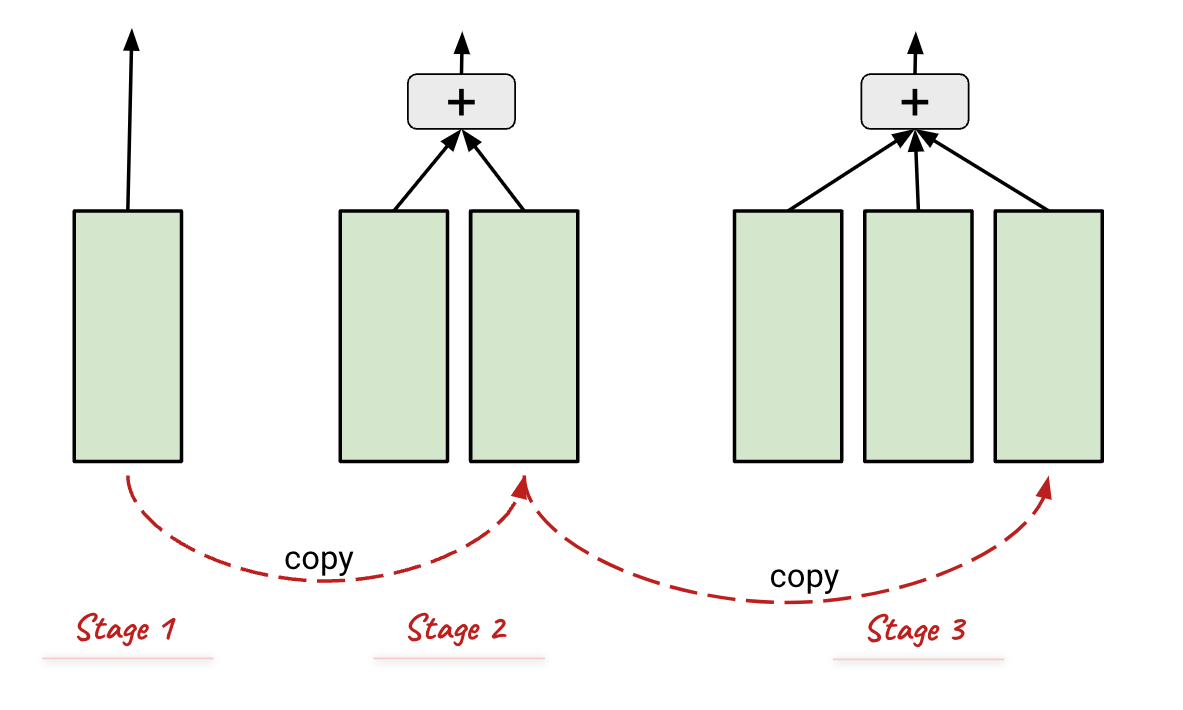}
  \caption{Stacking for additive models.}
  \label{fig:stacking-boosting}
\end{minipage}
\caption{Stacking for training language models and additive models. In each stage, a new transformer block/classifier is added, initialized with the parameters of the newest trained transformer block/classifier from the previous stage, and then trained for a certain number of steps.}
\label{fig:test}
\end{figure}
\fi

In the effort to reduce the massive costs of training these giant transformer models, greedy layer-wise pretraining has re-emerged as a very effective strategy in recent times. Specifically, a technique for initializing the new layers known as \emph{stacking} \cite{gong2019efficient, reddi2023efficient} has been shown to be very effective in speeding up training of deep transformer models. Stacking prescribes a heuristic for \emph{initializing} the newly added layers. Specifically, it prescribes that the newly added layers should be initialized by copying parameters from the previously trained layers.
\citet{gong2019efficient} proposed to double the model depth at each time by stacking an exact copy of the current model on top. \citet{reddi2023efficient} argue that doubling the model depth may be suboptimal and a better approach is gradual stacking where a few new layers (say 3-4) are added during each stage. These layers are initialized by copying the top most layers from the existing model. See Figure~\ref{fig:stacking-lm} for an illustration of this technique for training deep transformer models.

\if\arxiv0
\begin{wrapfigure}{r}{0.5\textwidth}
    \begin{center}
    \includegraphics[width=0.4\textwidth,height=4cm]{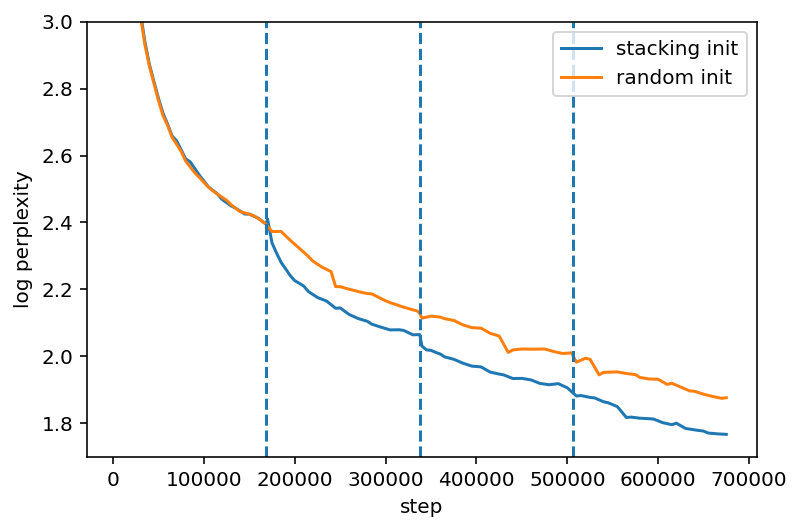}
    \end{center}
    \caption{\small Stacking init vs random init for stagewise training of BERT Base model. Four stages are used with 168,750 steps in each stage. Stage boundaries are marked by vertical dashed lines. Stacking init provides a clear benefit over random init.}
    \label{fig:bert-base}
\end{wrapfigure}
\fi
\if\arxiv1
\begin{figure}[!h]
    \centering
    \includegraphics[width=0.7\textwidth]{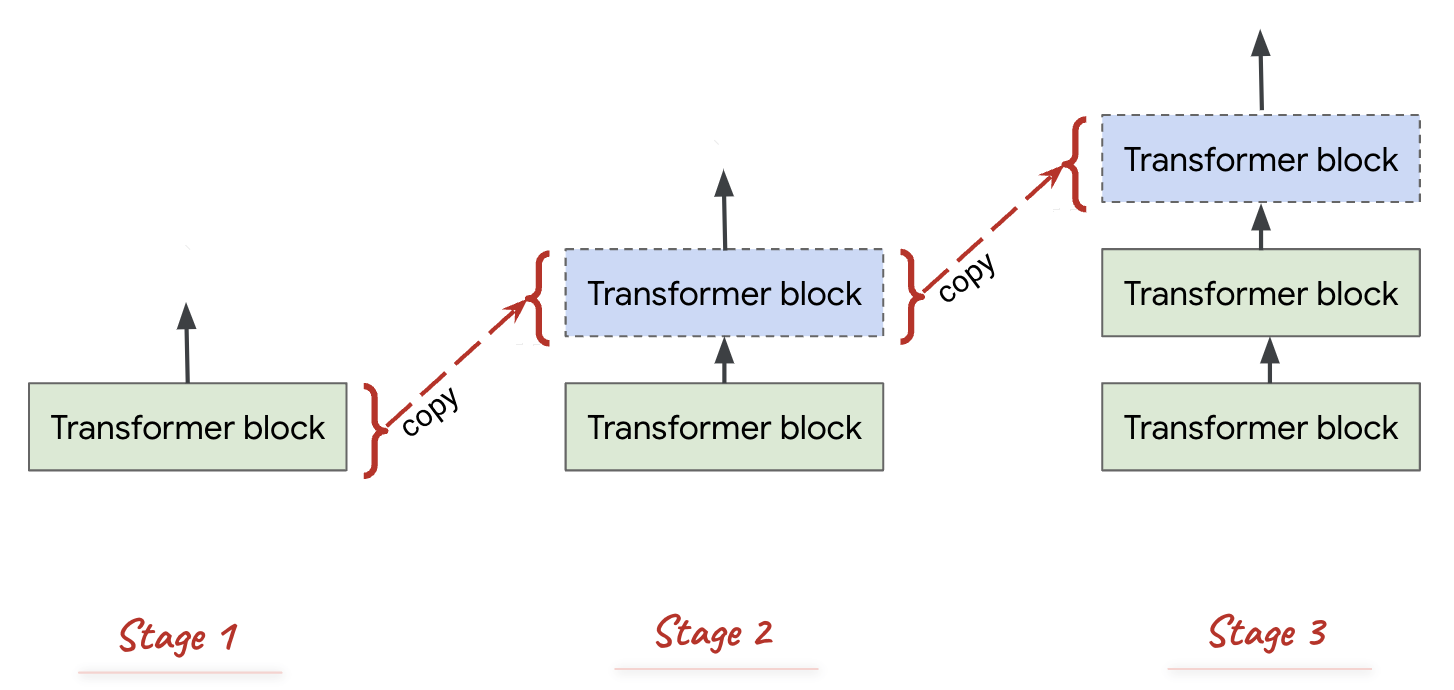}
    \caption{Stacking for stagewise training language models. In each stage, a new transformer block is added, initialized with the parameters of the top block from the previous stage, and then trained for a certain number of steps.}
    \label{fig:stacking-lm}
\end{figure}
\fi

The classical greedy layer-wise pretraining strategy doesn't have a specific prescription for initializing the new layers. In general, they're initialized randomly in some standard fashion. Stacking initialization provides a clear benefit over random initialization: Figure~\ref{fig:bert-base} shows one example of this effect, for training the BERT Base \cite{devlin} model with 4 stages of stagewise training.

\if\arxiv1
\begin{figure}[!h]
    \centering
    \includegraphics[width=0.55\textwidth]{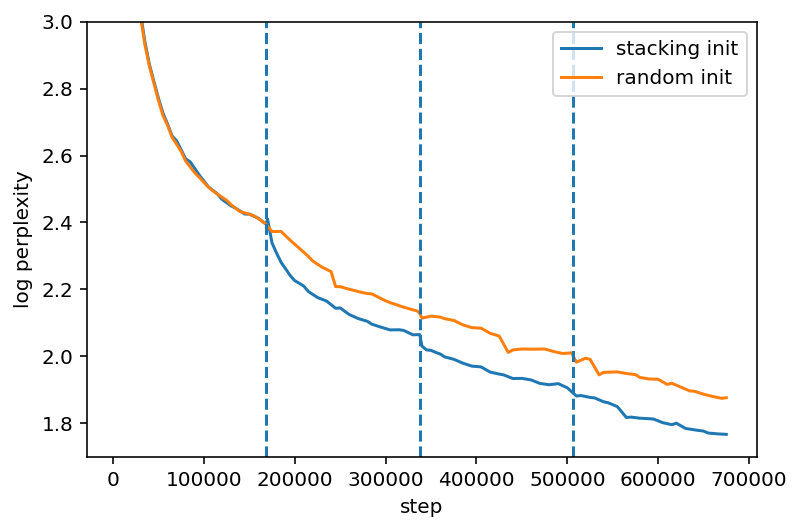}
    \caption{Stacking init vs random init for stagewise training of BERT Base model. Four stages are used with 168,750 steps in each stage. Stage boundaries are marked by vertical dashed lines. Stacking init provides a clear benefit over random init.}
    \label{fig:bert-base}
\end{figure}
\fi

Structurally, greedy layer-wise pretraining resembles another classical technique, viz. \emph{boosting}. In boosting, an additive ensemble of classifiers is constructed via greedy stagewise training of classifiers in the same greedy manner. Boosting algorithms such as AdaBoost \cite{freund1997decision} and Gradient Boosting \cite{friedman2001greedy} have found tremendous practical application, especially when using decision trees as base classifiers (e.g. XGBoost \cite{chen-guestrin}). A heuristic similar to stacking has also found practical application in boosting algorithms. The heuristic is to initialize each new classifier (e.g. a decision tree) by copying over the just-trained classifier and then updating it using new training data. This process is illustrated in Figure~\ref{fig:stacking-boosting}. Due to the similarity with stacking for training deep transformer models, in the rest of the paper we use ``stacking'' to also refer to this initialization strategy in the context of boosting.
\if\arxiv1
\begin{figure}[!h]
    \centering
    \includegraphics[width=0.6\textwidth]{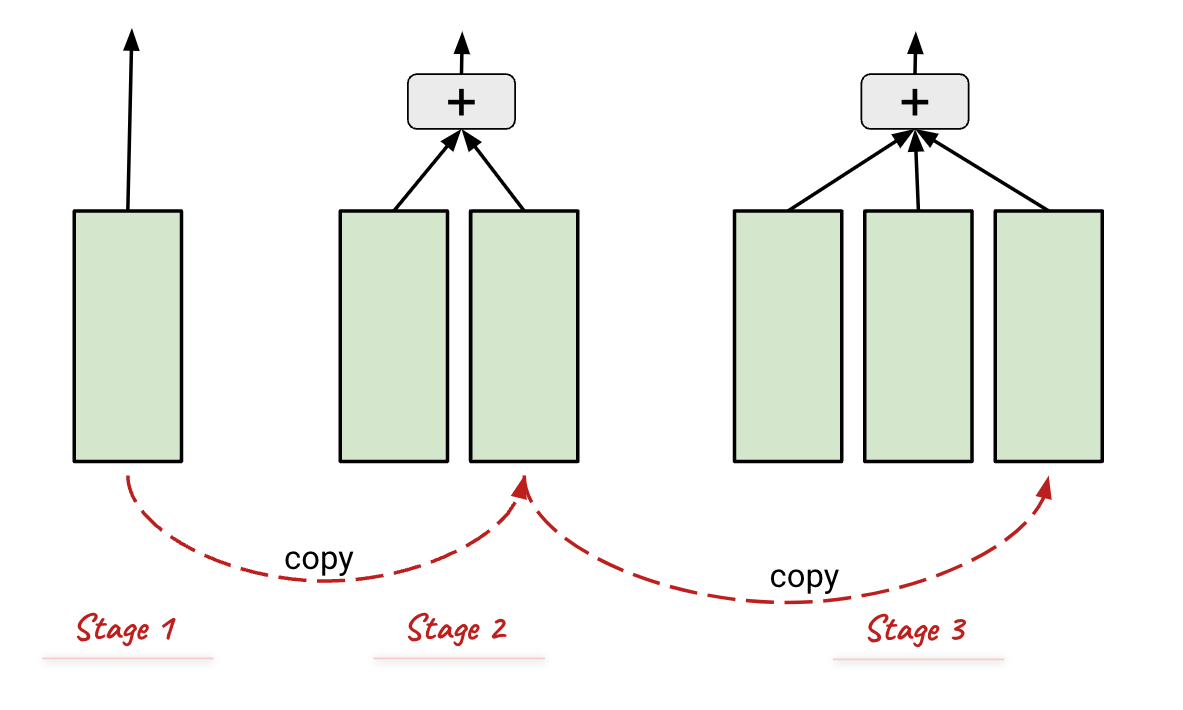}
    \caption{Stacking for boosting. In each stage, a new classifier is added, initialized with the parameters of the last trained classifier from the previous stage, and then trained for a certain number of steps.}
    \label{fig:stacking-boosting}
\end{figure}
\fi

While stacking based methods lead to impressive speed up in training of transformer models and additive models in boosting, we currently do not have a good theoretical understanding of why this is the case. Recently, \citet{reddi2023efficient} provided a theoretical explanation based on the assumption that each transformer block is a good few-shot learner. This assumption, along with a few others, are then used to conclude copying parameters as stacking does leads to fast learning. However, the assumptions made in the paper are fairly strong and hard to verify. 

In this work, we make progress on developing a theoretical understanding of the efficacy of stacking by studying it from an optimization perspective. In particular, our main contribution is that when viewed from the perspective of function optimization,  stacking speeds up stagewise training by enabling a form of the \emph{accelerated} gradient descent method (AGD) developed by Nesterov \citep{nesterov1983method}. In other words, each stage of the stacking-initialized stagewise training procedure will reduce the training loss at an accelerated rate. 

In contrast, we also show that without using any form of initialization, or in other words, initializing the new block/classifier to implement the zero function, stagewise training simply recovers usual (non-accelerated) gradient descent, whereas random initialization recovers stochastic gradient descent on a smoothed version of the loss. Hence, stacking initialization accelerates stagewise training over  zero or random initialization. In more detail, our contributions are as follows:

\if\arxiv0
\begin{enumerate}[leftmargin=*]
\else
\begin{enumerate}
\fi
    \item We propose a general theoretical framework towards learning a prediction function $F$ via an ensemble, i.e., a sequence of functions $(f_1, f_2, \ldots, f_T)$ in a greedy stagewise manner. The generality of our framework lets us unify classical approaches such as {\em boosting} \citep{freund1997decision, friedman2001greedy} that build the ensemble in an additive manner, and modern approaches that build the ensemble via stagewise training of \emph{residual} function compositions (e.g. ResNets \cite{kaiminghe} and Transformer models \cite{vaswani2017attention}).
    
    \item Our proposed framework lets us formally establish the connection between various initialization strategies used for building the ensemble and the convergence properties of the resulting overall learning procedure. In particular, we show that the zero initialization strategy recovers the vanilla functional gradient descent algorithm, for both the additive (i.e. boosting) and residual compositional forms of learning, whereas random initialization recovers \emph{stochastic} functional gradient descent (on a smoothed loss) for both types of models. Furthermore, in the case of additive models, the use of the popular stacking initialization exactly recovers Nesterov's accelerated functional gradient descent. The consequence is that for $T$ stages of boosting with stacking initialization, loss reduces at a rate of $O(T^{-2})$ for smooth losses, or $\exp(-\Omega(T/\sqrt{\kappa}))$ for smooth and strongly-convex losses with condition number $\sqrt{\kappa}$, as opposed to rates of $O(T^{-1})$ and $\exp(-\Omega(T/\kappa))$ respectively for zero initialization.

    \item For the case of compositional models, we show that stacking initialization results in updates that look remarkably similar to Nesterov's accelerated functional gradient descent. Proving an accelerated rate in the general non-parametric functional setting seems intractable, so we analyze stacking in a special parametric setting of deep linear networks with a convex loss function. In this setting we prove (Theorem~\ref{theorem:stronglyconvexconvergence}) that the stacking initialization quantitatively leads to the same kind of convergence benefits over vanilla gradient descent as is observed for Nesterov's accelerated method. At the core of our proof is a novel potential function based analysis of Nesterov's method with errors in the momentum term that may be of independent interest (c.f. Lemma~\ref{lemma:stronglyconvexpotentialarg}).
    
    \item We perform proof-of-concept experiments (in \if\arxiv0{Appendixes~\ref{sec:experiments_main} and \ref{app:additional-expts}}\else{Section~\ref{sec:experiments_main}}\fi{}) to validate our theory on synthetic and real world data. 
\end{enumerate}

\subsection{Related work}
\label{sec:related}

Boosting is a classical technique for constructing additive ensembles via greedy stagewise training, and has a long and rich history of work. We refer the interested reader to the excellent textbook of \citet{freundschapirebook} for the literature on this topic.

The idea of training deep residual networks in a layer wise manner has been explored in many prior works. In earlier studies \citep{hinton2006fast, bengio2006greedy} the focus was on greedily adding trained layers to the model while keeping the bottom layers frozen followed by a final fine tuning step where the entire network is trained. In recent years progressive or gradual stacking \citep{gong2019efficient, gu2020transformer, shen2022staged, reddi2023efficient} has emerged as a powerful way to train deep networks especially transformer based architectures.

The empirical insight of \citet{gong2019efficient} was that the attention patterns in neighboring layers of trained transformer models show remarkable similarity. Hence, by copying the parameters from the previous layer one is providing a better initialization for the optimization procedure. As mentioned previously, \citet{reddi2023efficient} developed the gradual stacking approach based on the assumption that the trained transformer blocks are good few-shot learners, and showed that gradual stacking leads to significant wallclock improvements during training.

%% file: src/fgd.tex
\paragraph{Preliminaries.}
We consider a fairly general supervised learning setting. Denote the input space by $\mathcal{X}$ and the output space by $\mathcal{Y}$. Examples $(x, y) \in \mathcal{X} \times \mathcal{Y}$ are drawn from a distribution $D$ (which may simply be the empirical data distribution in the case of empirical risk minimization). We aim to model the input-output relationship via \emph{predictions} in $\mathbb{R}^d$, for some dimension parameter $d$. Given an example $(x, y) \in \mathcal{X} \times \mathcal{Y}$ the quality of a prediction $\hat{y} \in \mathbb{R}^d$ is measured via a loss function $\ell: \mathbb{R}^d \times \mathcal{Y} \to \mathbb{R}$. Predictions are computed using functions $f: \mathcal{X} \to \mathbb{R}^d$. We will assume that the predictor functions $f$ are square integrable with respect to $D$, i.e. $\E_{(x, y) \sim D}[\|f(x)\|^2] < \infty$. The space of such functions forms a Hilbert space, denoted $\mathcal{L}_2$, with the inner product defined as $\langle f, g\rangle = \E_{(x, y) \sim D}[\langle f(x), g(x)\rangle]$. Unless specified otherwise, all functions in the subsequent discussion will be assumed to be in $\mathcal{L}_2$. The loss function can then naturally be extended to predictor functions $f \in \mathcal{L}_2$ by defining, with some abuse of notation, $\ell(f) := \E_{(x, y) \sim D}[\ell(f(x), y)]$. The goal of training is to obtain a function $f \in \mathcal{L}_2$ that minimizes $\ell(f)$.

In the rest of this section, we perform the analysis in a purely functional setting, which affords a convenient analysis. However, we note that, in practice, functions are parameterized (say by neural networks) and hence update rules for functions may not always be realizable via the specific parameterization used. The functional setting allows us to sidestep realizability issues and to focus on the conceptual message that stacking initialization enables accelerated updates.

We now define a general ensemble learning setup within the above setting. In this setup, we aim to approximate the minimizer of $\ell$ on $\mathcal{L}_2$ via an \emph{ensemble}, which is a sequence of functions $(f_1, f_2, \ldots, f_T)$, where $T > 0$ is a given parameter defining the size of the ensemble.
The functions in the ensemble are typically ``simple'' in the sense that they are chosen from a class of functions that is easy to optimize over.
A predictor function can be obtained from an ensemble $(f_1, f_2, \dots, f_T)$ by aggregating its constituent functions into a single function $F_T: \mathcal{X} \to \mathbb{R}^d$.
The loss of an ensemble can then be defined (again with some abuse of notation) in terms of its aggregation as $\ell((f_1, f_2, \ldots, f_T)) := \ell(F_T)$.
Two specific aggregation operators we consider are the following:
\begin{enumerate}[leftmargin=*]
    \item \textbf{Addition:} (E.g. boosting.) This is a summation over ensemble outputs: $F_T = f_1 + f_2 + \cdots f_T$.
    \item \textbf{Residual composition:} (E.g. deep residual neural networks.) This is a composed function $F_T = (I + f_T) \circ (I + f_{T-1}) \circ \cdots \circ (I + f_1)$, where the domain is $\mathcal{X} = \mathbb{R}^d$ and $I: \mathbb{R}^d \to \mathbb{R}^d$ is the identity mapping.
\end{enumerate}

\paragraph{Greedy stagewise training.}
\emph{Stagewise training} is a simple greedy procedure to train ensembles in a progressive manner. Suppose we have already obtained a (partial) ensemble $(f_1, f_2, \ldots, f_t)$. Then, the next function in the ensemble, $f_{t+1}$, is \emph{ideally} obtained by minimizing the loss of the new ensemble, i.e. $f_{t+1} = \arg\min_{f} \ell((f_1, f_2, \ldots, f_t, f))$. 

However, in practice, this ideal is hard to implement, and instead two heuristics are commonly used: (a) the new function to be trained is \emph{initialized} in some carefully chosen manner, and (b) the optimization above is done using \emph{early stopping}, i.e. a few steps of gradient descent, which ensures that the new function stays close to initialization. We analyze these heuristics in a functional optimization setting as follows.

First, we assume that the function $f_{t+1}$ to be trained is initialized at some carefully chosen value $f_{t+1}^0$. For notational convenience, we denote the aggregation of the ensemble $(f_1, f_2, \ldots, f_t, f_{t+1}^0)$ by $F_{t+1}^0$ and that of the generic ensemble $(f_1, f_2, \ldots, f_t, f)$ by $F$. 

Next, we note that an exact analysis for early stopping quickly becomes technically intractable. Instead, for a theoretical analysis, we \emph{model} the heuristic of early stopping by using $\ell_2$ regularization around the initialization and linearizing the loss near the initialization, as follows. It is known (see, e.g. \citep[Section 7.8]{Goodfellow-et-al-2016}) that early stopping acts as a form of $\ell_2$ regularization  which ensures that the trained function $f_{t+1}$ remains close to its initialization, $f_{t+1}^0$, which implies that $F_{t+1}$ remains close to $F_{t+1}^0$. Thus, early stopping can be modeled as minimizing $\ell(F) + \frac{\lambda}{2}\|F - F_{t+1}^0\|^2$, for some regularization parameter $\lambda$. Further, since the trained function remains close to the initialization, we also approximate $\ell(F)$ by its linearization around the initialization:
$\ell(F) \approx \ell(F_{t+1}^0) + \langle \nabla \ell(F_{t+1}^0), F - F_{t+1}^0\rangle$. Here, $\nabla \ell(\cdot)$ is the Fr\'echet derivative, and $\langle \cdot, \cdot \rangle$ denotes the inner product in $\mathcal{L}_2$. 
These considerations lead to the following key modeling assumption.
\begin{assumption}
The result of the early stopped training is given by
\begin{align*}
F_{t+1} = \arg\min_{F \in \mathcal{L}_2} \ell(F_{t+1}^0) + \langle \nabla \ell(F_{t+1}^0), F - F_{t+1}^0\rangle + \frac{\lambda}{2}\|F - F_{t+1}^0\|^2.
\end{align*}
In other words, 
\begin{equation} \label{eq:F_t+1-formula}
  F_{t+1} = F_{t+1}^0 - \frac{1}{\lambda}\nabla \ell(F_{t+1}^0).  
\end{equation}
\end{assumption}

We can now consider specific initialization strategies (i.e. zero initialization, random initialization, and stacking initialization) in the context of additive and residual compositional models and see how these initializations lead to various forms of functional gradient descent.

\paragraph{Stagewise training with zero initialization recovers functional gradient descent.}
First, consider stagewise training where functions are initialized to be zero functions, i.e. $f_{t+1}^0 = 0$. 
It is easy to see that with this initialization, for both additive and residual compositional models, we have $F_{t+1}^0 = F_t$. Thus, from \eqref{eq:F_t+1-formula}, we have that the updated ensemble's predictor can be written as
\[F_{t+1} = F_t - \frac{1}{\lambda}\nabla \ell(F_t).\]
This exactly describes functional gradient descent with step size $\frac{1}{\lambda}$. In the additive setting this is well-known: indeed, boosting can be seen as a functional gradient descent \cite{mason}. The result for the residual compositional setting appears to be new.

\paragraph{Stagewise training with random initialization recovers stochastic functional gradient descent on smoothed loss.} We now consider stagewise training where functions are initialized randomly, i.e. $f_{t+1}^0$ is a randomly drawn function, independent of all randomness up to stage $t$. In the following, we will assume that $\E[f_{t+1}^0] = 0$, where the $0$ on the RHS denotes the zero function. With this initialization, for both additive and residual compositional models, we have $F_{t+1}^0 = F_t + g_t$, where $g_t = f_{t+1}^0$ for additive models, and $g_t = f_{t+1}^0 \circ F_t$ for residual compositional models. In either case, note that $\E[g_t] = 0$. Now define the loss functional $\ell_t(F) := \E[\ell(F + g_t)]$. Since $\E[g_t] = 0$, we can interpret $\ell_t$ as a randomized smoothing of $\ell$, similar to convolving with a Gaussian. Then, from \eqref{eq:F_t+1-formula}, we have that the updated ensemble's predictor can be written as
\[F_{t+1} = F_t - \frac{1}{\lambda}(\nabla \ell(F_t + g_t) - \lambda g_t).\] 
Now, note that 
\[\E[\nabla \ell(F_t + g_t) - \lambda g_t] = \nabla\ell_t(F_t).\]
Or in other words, the above update can be seen as a \emph{stochastic} functional gradient descent step on the smoothed loss function $\ell_t$.

\noindent
\paragraph{Stagewise training with stacking initialization recovers accelerated functional gradient descent.}
We now consider stagewise training where functions are initialized in a stacking-like fashion with $f_{t+1}^0 = f_t$, which we will refer to as the \emph{stacking initialization}.
When the ensemble aggregation operator is addition, we have $F_{t+1}^0 = f_t + F_t = F_t + (F_t - F_{t-1})$ and hence \eqref{eq:F_t+1-formula} implies that the updated ensemble's predictor is
\[F_{t+1} = F_t + (F_t - F_{t-1}) - \frac{1}{\lambda} \nabla \ell(F_t + (F_t - F_{t-1})).\]
The above formula essentially describes Nesterov's accelerated gradient descent, which has the following update rule:
\begin{equation} \label{eq:nesterov}
    F_{t+1} = F_t + \beta(F_t - F_{t-1}) - \frac{1}{\lambda} \nabla \ell(F_t + \beta(F_t - F_{t-1})).
\end{equation}
Here, $\beta \in [0, 1)$ is a constant that can depend on $t$. In fact, we can \emph{exactly} recover Nesterov's accelerated gradient descent if we modify the stacking initialization to $f_{t+1}^0 = \beta f_t$. Thus, stacking enables accelerated descent for training additive models.

When the ensemble aggregation operator is residual composition, stagewise training with the stacking initialization $f_{t+1}^0 = f_t$ results in
\[F_{t+1}^0 = (I + f_t) \circ F_t = F_t + f_t \circ F_t.\]
Equation \eqref{eq:F_t+1-formula} therefore implies the updated ensemble's predictor is
\begin{equation} \label{eq:stacking-update}
F_{t+1} = F_t + f_t \circ F_t - \frac{1}{\lambda}\nabla \ell(F_t + f_t \circ F_t).
\end{equation}
In contrast, Nesterov's update rule \eqref{eq:nesterov}, and the fact that for residual compositional models, 
\[F_t - F_{t-1} = (I + f_t) \circ F_{t-1} = f_t \circ F_{t-1},\] 
yields the following equation for $F_{t+1}$:
\begin{equation} \label{eq:nesterov-comp}
    F_{t+1} = F_t + \beta f_t \circ F_{t-1} - \frac{1}{\lambda} \nabla \ell(F_t + \beta f_t \circ F_{t-1}).
\end{equation}
Comparing \eqref{eq:stacking-update} and \eqref{eq:nesterov-comp}, barring the minor difference in $\beta$ parameters, which can be easily rectified as in the case of the additive models by setting $f_{t+1}^0 = \beta f_t$, the major difference is that $f_t \circ F_t$ replaces $f_t \circ F_{t-1}$. Although possibly intractable to prove formally, we believe that the updates in \eqref{eq:stacking-update} also provide an accelerated convergence rate, since we expect $F_{t-1}$ to be close to $F_t$ as iterates converge to the optimal function. 

In the following section, we show that in certain deep linear networks, the above intuition is indeed correct and provide a rigorous proof that stacking provides an accelerated convergence rate.

%% file: src/main_results.tex
To demonstrate that stacking can provide a provably accelerated rate of convergence, we now turn to studying the narrower setting of training deep residual linear networks, which are fully connected feedforward neural networks without non-linear activations and with residual connections. Such networks are a common subject of study in the theory of deep learning \cite{saxe_exact_13, kawaguchi_deep_16, hardt-ma}.
As they have no non-linear components, deep linear networks effectively compute a linear functions, albeit via a parametrization as a product of the weight matrices.

\paragraph{Setup.}
Consider again the general supervised learning setting from Section~\ref{sec:functional} and suppose, as is often the case in modern neural networks, that examples consist of inputs $x \in \reals^d$ and outputs $y \in \cY$. The loss function $\ell: \reals^d \times \cY \to \reals$ is assumed to be convex in the first argument. Let the samples be drawn from a distribution $D$ over $\reals^d \times \cY$. Then the expected loss of the linear predictor $x \mapsto Wx$ for a matrix $W \in \reals^{d \times d}$ is (with some abuse of notation) $\ell(W) := \E_{(x, y) \sim D}[\ell(Wx, y)]$. In the following, we will assume the expected loss $\ell(W)$ is $L$-smooth and $\mu$-strongly convex in $W$, by which we mean that the following inequalities hold for any $W, V \in \reals^{d \times d}$:
\begin{align*}
\ell(W) + \langle \nabla \ell(W), V-W\rangle + \frac{\mu}{2}\norm{W - V}^2 \leq \ell(V) \leq \ell(W) + \langle \nabla \ell(W), V-W\rangle + \frac{L}{2}\norm{W - V}^2.
\end{align*}
Here, for matrices $W, V \in \reals^d$, $\langle W, V \rangle = \text{Tr}(W^\top V)$, and $\norm{W}$ is the Frobenius norm of $W$.
The \emph{condition number} $\kappa$ of the loss is defined as $\kappa := \frac{L}{\mu}$.

The deep residual neural networks we consider have $t$ layers with weight matrices $w_1, w_2, \ldots, w_t$, and the function they compute is $x \mapsto W_tx$, where
\[W_t := (I + w_t) (I + w_{t-1}) \dots (I + w_1).\]
Here, $I \in \reals^{d \times d}$ is the identity matrix providing the residual connection. The expected loss of the neural network described above on the data is $\ell(W_t)$.

\paragraph{Derivation of stacking updates.}
Suppose we train the deep residual linear network described above using stacking initialization, but incorporating $\beta$-scaling: i.e., to train the $(t+1)$-th layer, its weight matrix is initialized to $w_{t+1}^0 = \beta w_t$, for some constant $\beta \in [0, 1]$, and then trained. Following the exact same steps as in the derivation of stacking updates in the functional setting of Section~\ref{sec:functional}, we end up with the following formula for $W_{t+1}$:
\[W_{t+1} = W_t + \beta w_t W_t - \frac{1}{\lambda}\nabla \ell(W_t + \beta w_t W_t).\]

When $W_{t-1}$ is non-singular, we have $w_t = W_tW_{t-1}^{-1} - I = (W_t - W_{t-1})W_{t-1}^{-1}$, so the above equation can be rewritten as
\begin{align}
W_{t+1} = &\; W_{t} + \beta (W_{t} - W_{t-1})W_{t-1}^{-1} W_{t} - \frac{1}{\lambda}\nabla \loss(W_{t} + \beta (W_{t} - W_{t-1})W_{t-1}^{-1} W_{t}).
    \label{eq:linearupdate}
\end{align}
As previously noted, \eqref{eq:linearupdate} differs from Nesterov's AGD method in form: Nesterov's AGD updates would be
\begin{align}
W_{t+1} &= W_{t} + \beta (W_{t} - W_{t-1}) - \frac{1}{\lambda}\nabla \loss(W_{t} + \beta (W_{t} - W_{t-1})).
\label{eq:nesterov-update}
\end{align}

\paragraph{Accelerated convergence for stacking updates.} 
Despite the differences with Nesterov's method, we can show that stacking notably still yields a provably accelerated convergence rate.
Let $W^* := \arg\min_W \ell(W)$. Suppose that $W^*$ is non-singular, i.e. its smallest singular value $\sigma_{\min}(W^*) > 0$.
Theorem~\ref{theorem:stronglyconvexconvergence} shows that as long as the first two layers are initialized so that $W_1$ and $W_2$ are close to optimal, stacking results in a suboptimality gap of $\exp(-\tilde\Omega(T/\sqrt{\kappa}))$ after $t$ stages of stacking. This is of the same order as the rate obtained by Nesterov's acceleration; note that, in comparison, stagewise training with zero initialization results would result in a suboptimality gap of $\exp(-\tilde\Omega(T/\kappa))$.
\begin{restatable}{theorem}{stronglyconvexconvergence}
\label{theorem:stronglyconvexconvergence}
Consider stagewise training with stacking initialization of a deep residual linear network in the setup described above with $\beta = \tfrac{\sqrt{\kappa}-1}{\sqrt{\kappa}+1}$ and $\lambda = \tfrac 1 L$. Suppose that the first layer weights are initialized so that $W_1 = V_0 - \frac{1}{L}\nabla \ell(V_0)$, where $V_0 \in \reals^{d \times d}$ satisfies $\ell(V_0) - \ell(W^*) \leq \delta$ for $\delta \asseq \tfrac{\mu\alpha^2\sigma_{\min}(W^*)^2}{256\beta^2 d}$ and $\alpha^{-1} \asseq (\kappa - 1)\sqrt{2 \sqrt{\kappa}( \kappa - 1)(\sqrt{\kappa} - 3)}$. Then after $T$ stages of stacking, we have 
\[\loss(W_T) - \loss(W^*) \leq \exp(-\tilde\Omega(T/\sqrt{\kappa})).\]
\end{restatable}

The $\tilde\Omega(\cdot)$ notation above hides polylogarithmic dependence on the problem parameters for clarity of presentation. Precise expressions can be found in the proof.
The primary insight behind the result of Theorem~\ref{theorem:stronglyconvexconvergence} is that Nesterov's accelerated gradient method is relatively robust to perturbations in its update rules.
This robustness is formalized below in Lemma~\ref{lemma:stronglyconvexpotentialarg}. The lemma is described in a fairly general, standalone setting since it may be of independent interest.
\begin{restatable}[Robustness of Nesterov's accelerated gradient method]{lemma}{stronglyconvexpotentialarg}
\label{lemma:stronglyconvexpotentialarg}
Let $\cF$ be a Hilbert space and $\loss: \cF \to \reals$ an $\lipschitz{}$-smooth and $\mu$-strongly convex function to be minimized on $\cF$.
Consider the iterates $x_0, y_0, \dots, x_T, y_T \in \cF$ with $x_0 = y_0$ chosen arbitrarily, and the update rules
\begin{align}
    y_{t+1}& \asseq x_t - \frac{1}{\lipschitz{}}\nabla \loss(x_t) \nonumber \\
    x_{t+1}& \asseq y_{t+1} + \beta(y_{t+1} - y_t) + \Delta_{t+1} \label{eq:nesterov-1}
\end{align}
where $\beta := \frac{\sqrt{\kappa}-1}{\sqrt{\kappa}+1}$, $\kappa = \tfrac{\lipschitz}{\mu}$, and $\Delta_1, \Delta_2, \dots, \Delta_{T-1} \in \cF$ are error terms such that 
$\norm{\Delta_{t}} \in O(\kappa^{-2} \norm{y_{t} - y_{t-1}})$
for all $t \in [T-1]$. Then the convergence rate of the iterates to a suboptimality gap of $\epsilon$ is of order $\tilde{O}(\sqrt{\kappa} \log(\kappa / \epsilon))$. Specifically, for any $T \geq 2$, $$\loss(y_T) - \loss(x^*) \leq (\tfrac{2 \sqrt{\kappa} - 2}{2 \sqrt{\kappa} - 1})^T \para{\tfrac {\lipschitz{}} 2  +\tfrac{8 \sqrt{\kappa} \mu}{4 (4 \sqrt{\kappa} - 3)}} \norm{x_0 - x^*}^2$$
and
$$\loss(y_T) - \loss(x^*) \leq (\tfrac{2 \sqrt{\kappa} - 2}{2 \sqrt{\kappa} - 1})^T \tfrac{8 \sqrt{\kappa}}{4  \sqrt{\kappa} - 3} (\loss(y_0) - \loss(x^*)).$$
\end{restatable}

We will apply Lemma~\ref{lemma:stronglyconvexpotentialarg} using the correspondence $x_0 = y_0 = W_0 = V_0$, and for all $t \geq 1$, $y_t = W_t$ and $x_t = W_t + \beta(W_t - W_{t-1})W_{t-1}^{-1}W_t$. Lemma~\ref{lemma:stronglyconvexpotentialarg}
implies that even though stagewise training with stacking differs from Nesterov's method, their similar form allows us to express the former as a perturbation of the latter.
In particular, if we write our stacking update \eqref{eq:linearupdate} for deep residual linear networks as a perturbation of Nesterov's method, as in \eqref{eq:nesterov-1}, the perturbation term is exactly %
\begin{align}
\label{eq:perturb}
    \Delta_{t} = \beta (W_{t} - W_{t-1}) W_{t-1}^{-1} W_{t} - \beta (W_t - W_{t-1})
\end{align}
for all $t \in [2, T-1]$.
That is, if we examine the sequence of iterates $W_1, \dots, W_T$ produced by stagewise training with stacking initialization, the term $\Delta_t$ is a measure of the disagreement between the realized iterate $W_{t+1}$ and what Nesterov's method says the iterate $W_{t+1}$ should be conditioned on the iterates $W_1, \dots, W_{t}$ from previous timesteps. 
We note that, even when these perturbation terms $\Delta_1, \dots, \Delta_{T-1}$ are small in norm, it is possible for Nesterov's method to describe an iterate sequence that diverges significantly in norm from the iterates realized by stagewise training.

Rewriting \eqref{eq:perturb} as
$\Delta_t = \beta (W_t - W_{t-1})W_{t-1}^{-1}(W_t - W_{t-1})$,
we immediately see that in order to satisfy the requirement of Lemma~\ref{lemma:stronglyconvexpotentialarg} that $\norm{\Delta_{t}} \in O(\kappa^{-2} \norm{W_{t} - W_{t-1}})$, we simply need $\norm{\beta W_{t-1}^{-1}(W_t - W_{t-1})} \in O(\kappa^{-2})$. This is satisfied when $W_t \approx W_{t-1}$ and $W_{t-1}$ is reasonably non-singular.
A sufficient condition for this is that the iterates are sufficiently close to the ground-truth solution $W^*$ and $W^*$ is non-singular, which explains the conditions of Theorem~\ref{theorem:stronglyconvexconvergence}.

We now prove Theorem~\ref{theorem:stronglyconvexconvergence} formally.

\begin{proof}[Proof of Theorem~\ref{theorem:stronglyconvexconvergence}]

As described above, the deep linear networks $W_1, \dots, W_T$, as defined in \eqref{eq:linearupdate}, can be written as iterates of a variant of Nesterov's acceleration \eqref{eq:nesterov-1}, where we set the gradient step size to be the usual $\lambda = \tfrac 1 L$ and the stacking parameter $\beta$ to match the usual momentum parameter setting of $\tfrac {\sqrt{\kappa} - 1} {\sqrt{\kappa}+1}$.

\paragraph{Sufficient claim.}
To prove the main result, it suffices to show that $\norm{\Delta_{t}} \leq \alpha \norm{W_{t} - W_{t-1}}$ for all $t \in [T-1]$.
With this claim Lemma~\ref{lemma:stronglyconvexpotentialarg} immediately gives a convergence rate of
\begin{align*}
    \loss(W_T) - \loss(W^*)
    \leq (\tfrac{2 \sqrt{\kappa} - 2}{2 \sqrt{\kappa} - 1})^{T} \para{\tfrac{\lipschitz{}}{2} +\tfrac{8 \sqrt{\kappa} \mu}{4 (4 \sqrt{\kappa} - 3)}} \norm{W_1 - W^*}^2.
\end{align*}
This gives the desired result $\loss(W_T) - \loss(W^*) \in \exp(-\Omega(\tfrac T {\sqrt{\kappa}} + \log(\tfrac T {\sqrt{\kappa}})))$ as
\begin{align*}
  \tfrac 1C(  \loss(W_T) - \loss(W^*)) 
  \leq (1 - \tfrac{1}{2 \sqrt{\kappa} - 1})^{T}  
\leq \exp(-\Omega(\tfrac T {\sqrt{\kappa}}))
\end{align*}
where $C = \para{\tfrac{\lipschitz{}}{2} +\tfrac{8 \sqrt{\kappa} \mu}{4 (4 \sqrt{\kappa} - 3)}} \norm{W_1 - W^*}^2 \in \exp(-\Omega(\log(\tfrac{T}{\sqrt{\kappa}})))$.

Thus, we need to show that $\norm{\Delta_{t}} \leq \alpha \norm{W_{t} - W_{t-1}}$ for all $t \in [T-1]$. For this, we use the following claim, where $\eta \asseq \sqrt{\frac{16\delta}{\mu}}$:
\begin{claim} \label{claim:closeness-to-bound}
Suppose that $\norm{W_{t-1} - W^*} \leq \eta$ and $\norm{W_t-W^*} \leq \eta$. Then $\norm{\Delta_{t}} \leq \alpha \norm{W_{t} - W_{t-1}}$.
\end{claim}
\begin{proof}[Proof of Claim~\ref{claim:closeness-to-bound}]
We have
$$\norm{\Delta_t} = \norm{\beta (W_t - W_{t-1})  (W_{t-1}^{-1} W_t - I)} \leq \beta \norm{W_t -W_{t-1}} \norm{W_{t-1}^{-1} W_t - I}.$$
We can further simplify right-most factor as follows:
\begin{align}
\norm{W_{t-1}^{-1} W_{t} - I} 
 &=  \norm{W_{t-1}^{-1}(W_{t} - W_{t-1})} \notag \\
 &\leq \norm{W_t - W_{t-1}} \norm{W_{t-1}^{-1}} & \text{(submultiplicative property)} \notag \\
 &\leq (\norm{W_t - W^*} + \norm{W_{t-1} - W^*}) \norm{W_{t-1}^{-1}}. & \text{(triangle inequality)} \label{eq:right-factor}
 \end{align}
 The following fact shows that $W_{t-1}$ is indeed far from singular as long as it is close enough to $W^*$.
\begin{fact}
\label{fact:inverse}
$W_{t-1}$ is invertible and $\norm{W_{t-1}^{-1}} \leq \frac{\sqrt{d}}{\sigma_{\min{}}(W^*) - \eta}$.
\end{fact}
\begin{proof}[Proof of Fact~\ref{fact:inverse}]
Since $\|\cdot\|$ is the Frobenius norm, we can upper bound the singular values of the matrix $W^* - W_{t-1}$ by $\sigma_{\max{}}(W^* - W_{t-1}) \leq \eta$, where we use $\sigma_{\max{}}(W)$  to denote the largest singular values of a matrix $W$.
We can then use Weil's inequality to argue that, for any $i \in [d]$, the difference between the $i$th largest singular values of $W_{t-1}$ and $W^*$ is upper bounded by $$\abs{\sigma_i(W_{t-1}) - \sigma_i(W^*)} \leq \sigma_{\max}(W_{t-1} - W^*) \leq \eta.$$
The smallest singular value of $W_{t-1}$ is thus at least $\sigma_{\min{}}(W_{t-1}) \geq \sigma_{\min{}}(W^*) - \eta$.
Since our choice of $\delta$ guarantees that $\eta \leq \sigma_{\min{}}(W^*)$, we have that $W_{t-1} > 0$ and is invertible.
This also implies that the largest singular value of $W_{t-1}^{-1}$ is at most $ \sigma_{\max{}}(W_{t-1}^{-1})  \leq \frac{1}{\sigma_{\min{}}(W^*) - \eta}$.
We can therefore upper bound the Frobenius norm of $W_{t-1}^{-1}$ as claimed by $$\norm{W_{t-1}^{-1}} \leq \sqrt{\sum_{i \in [d]} \sigma_i(W_{t-1}^{-1})^2} \leq \sqrt{d} \sigma_{\max{}}(W_{t-1}^{-1}) \leq \sqrt{d} \frac{1}{\sigma_{\min{}}(W^*) - \eta}.$$
\end{proof}

Using Fact~\ref{fact:inverse} and \eqref{eq:right-factor}, we can now bound
\begin{align*}
\norm{W_{t-1}^{-1} W_{t} - I}
 &\leq 2 \eta \sqrt{d} \frac{1}{\sigma_{\min}(W^*) - \eta}.
\end{align*}
Since $\eta = \sqrt{\frac{16\delta}{\mu}}$, we have by the definition of $\delta$ that  $2 \eta \sqrt{d} \frac{1}{\sigma_{\min}(W^*) - \eta} \leq \tfrac \alpha \beta$.
This concludes our proof of the claim as
$$\norm{\Delta_t} \leq \beta \cdot \norm{W_t -W_{t-1}} \norm{W_{t-1}^{-1} W_t - I} \leq \beta \cdot \tfrac \alpha \beta \|W_t - W_{t-1}\|.$$
\end{proof}

We can now complete the proof by showing the following claim via induction on $t$:
\begin{claim} \label{claim:induction}
We have $\norm{\Delta_{t}} \leq \alpha \norm{W_{t} - W_{t-1}}$ and $\norm{W_t - W^*} \leq \eta$.
\end{claim}

\begin{proof}[Proof of Claim~\ref{claim:induction}]
\if\arxiv1
We proceed with induction on $t$.

\item\fi
\paragraph{Base case: $t=1$.}
Note that $W_0 = V_0$. We have 
\[\tfrac{\mu}{2}\norm{W_0 - W^*}^2 \leq \ell(W_0) - \ell(W^*) \leq \delta,\]
which implies that $\norm{W_0 - W^*} \leq \sqrt{\frac{2\delta}{\mu}} \leq \eta$. Furthermore, since $W_1 = W_0 - \frac{1}{L}\nabla \ell(W_0)$ and $\ell$ is $L$-smooth, we have $\ell(W_1) \leq \ell(W_0)$, which implies, exactly as above, that $\norm{W_1 - W^*} \leq \eta$. Thus, by Claim~\ref{claim:closeness-to-bound}, we conclude that $\norm{\Delta_{1}} \leq \alpha \norm{W_{1} - W_{0}}$.

\if\arxiv1\item\fi
\paragraph{Inductive step.}
Fixing $t \in [2, T-1]$, and assume, as our inductive hypothesis, that $\norm{\Delta_\tau} \leq \alpha \norm{W_{\tau} - W_{\tau-1}}$ and $\norm{W_\tau - W^*} \leq \eta$ for all $\tau \in [t-1]$.
Due to this inductive hypothesis, we can invoke Lemma~\ref{lemma:stronglyconvexpotentialarg} to observe that for all $\tau \in [T-1]$:
\begin{align*}
\tfrac \mu 2 \norm{W_{t} - W^*}^2 &\leq \loss(W_t) - \loss(W^*) \\
&\leq (\tfrac{2 \sqrt{\kappa} - 2}{2 \sqrt{\kappa} - 1})^t \tfrac{8 \sqrt{\kappa}}{4  \sqrt{\kappa} - 3} (\loss(W_0) - \loss(W^*)) \\
&\leq  8\delta.
\end{align*}
Thus, we have $\|W_t - W^*\| \leq \sqrt{\frac{16\delta}{\mu}} = \eta$. Since $\norm{W_{t-1} - W^*} \leq \eta$ due to our induction hypothesis, we can now use Claim~\ref{claim:closeness-to-bound} to conclude that $\norm{\Delta_t} \leq \alpha \norm{W_t - W_{t-1}}$. This completes the inductive proof.

\end{proof}

\end{proof}

Theorem~\ref{theorem:stronglyconvexconvergence} presents a \emph{local} convergence result: it assumes that the initial two layers put the network in the vicinity of the optimal solution $W^*$. We can also provide a \emph{global} convergence result by using a small warmup period phase to stacking---that is, by training the first few (only a constant number) stages of the deep linear network without stacking.
\begin{restatable}[Corollary of Theorem~\ref{theorem:stronglyconvexconvergence}]{corollary}{deeplinearnetworks}
\label{corollary:deeplinearnetworks}
Consider stagewise training of a deep residual linear network in the setup described above with an initial warmup phase of zero initialization for $\tilde{O}(\kappa)$ stages followed by stacking initialization for the remaining stages. Then after $T$ total stages, we have 
\[\loss(W_T) - \loss(W^*) \leq \exp(-\tilde\Omega(T/\sqrt{\kappa}) + \tilde O(\sqrt{\kappa})).\]
\end{restatable}
\if\arxiv1

\begin{proof}
We first recall the standard convergence rate of gradient descent.
In the same setting as Lemma~\ref{lemma:stronglyconvexpotentialarg} where we minimize a $L$-smooth and $\mu$-strongly convex function on some Hilbert space $\cF$, we can define gradient descent iterates with the update rule $x_{t+1} = x_t - \tfrac 1 L \nabla \loss(x_t)$.
For the resulting sequence of iterates $x_0, \dots, x_T$, it is known that $\loss(x_T) - \loss(x^*) \leq \exp(-T / \kappa) (\loss(x_0) - \loss(x^*))$; see e.g. \cite{gupta2017potential}.

We can further recall that, in the stagewise training of deep residual linear networks, we can write networks as $W_{t+1} = W_t + w_{t+1}^0 W_t - \tfrac 1 L \nabla \loss(W_t + w_{t+1}^0 W_t)$.
In the initial stages where new layers are initialized with zero weights, i.e. $w_{t+1}^0 = 0$, we recover, as mentioned in Section~\ref{sec:functional}, a gradient descent on $\reals^{d \times d}$:
\begin{align*}
    W_{t+1} = W_t - \tfrac 1 L \nabla \loss(W_t).
\end{align*}

Putting these two pieces together, we have $\loss(W_{t}) - \loss(W^*) \leq \exp(-\tfrac{(t-1)}{\kappa}) (\loss(W_1) - \loss(W^*))$.
By setting 
$$T_0 \geq \kappa \log( \tfrac{\loss(W_1) - \loss(W^*)}{\delta})+2,$$
we have $\ell(W_{T_0-1}) - \ell(W^*) \leq \delta$. So by setting $V_0 = W_{T_0-1}$, and noting that $W_{T_0} = W_{T_0-1} - \frac{1}{L}\nabla \ell(W_{T_0-1})$, we can now apply the local convergence result of Theorem~\ref{theorem:stronglyconvexconvergence}, and conclude that performing $T'$ rounds of stacking after $T_0$ rounds of warm-start leads to the desired loss bound \begin{align*}
    \loss(W_{T_0 + T'}) - \loss(W^*) \leq \exp(-\tilde \Omega(T' /\sqrt{\kappa})) &= \exp(-\tilde \Omega(T /\sqrt{\kappa}) + O(T_0 / \sqrt{\kappa})) \\
    &= \exp(-\tilde \Omega(T /\sqrt{\kappa}) + \tilde O(\sqrt{\kappa})).
\end{align*}
\end{proof}

\paragraph{Key ingredient of proof: robustness of Nesterov's accelerated gradient descent method.}
We now turn to proving Lemma~\ref{lemma:stronglyconvexpotentialarg}.

\begin{proof}[Proof of Lemma~\ref{lemma:stronglyconvexpotentialarg}]
In this proof, we will assume that the norm of the perturbation term is bounded by \begin{align}
    \|\Delta_{t}\| \leq \alpha \|y_{t} - y_{t-1}\|,
    \label{eq:perturbationbound}
\end{align}
where $\alpha$ is defined in the same way as in Theorem~\ref{theorem:stronglyconvexconvergence}, namely 
\[\alpha^{-1} \asseq (\kappa - 1)\sqrt{2 \sqrt{\kappa}( \kappa - 1)(\sqrt{\kappa} - 3)} .\]
We define the coefficients $\tau := \frac{1}{\sqrt{\kappa}+1}$, $\gamma := \frac{1}{\sqrt{\kappa}-1}$, and $\rho = \frac{\mu \gamma}{4(4 + \gamma)} = \frac{\mu}{4 (4 \sqrt{\kappa} - 3)}$.
$\tau$ can be understood as a momentum parameter, $\gamma$ as (proportional to) the parameter of the exponential curve that is our convergence rate, and $\rho$ a penalty for the presence of the perturbations $\Delta_t$.
As is done in all proofs of Nesterov's method, we will define the iterates 
\[z_{t} := \tfrac{1}{\tau}x_{t} - \tfrac{1-\tau}{\tau}y_{t}\]
for all $t \in [T]$; we refer interested readers to \citet{zhu2017, ahn2022} for interpretations of these iterates.
We note that although our definition of $z_{T}$ depends on $x_T$, and by extension $\Delta_T$, this lemma's claim about $y_T$ does not directly depend at all on $\Delta_T$.
Thus, we will freely choose to set $\Delta_T = 0$, guaranteeing that $\norm{\Delta_T} \leq \alpha \norm{y_{T} - y_{T-1}}$.

To prove the statement of this lemma, it suffices to show the following claim.
\begin{claim}
\label{claim:potential}
The following function $\Phi$ is a potential function:
\begin{equation}
\label{eq:lyapunov}
    \Phi(t) = (1+\tfrac{\gamma}{2})^t[(\loss(y_t)  - \loss(x^*)) + \tfrac{\mu}{2}\|z_t - x^*\|^2 + 2\rho  \|y_t -x^*\|^2].
\end{equation}
That is, for all $t \in [T]$, $\Phi(t) - \Phi(t-1) \leq 0$.
\end{claim}
It suffices to show Claim~\ref{claim:potential} because the consequence that $\Phi(T) \leq \Phi(0)$ directly implies the lemma's first statement via
\begin{align*}
    (1 + \tfrac{\gamma}{2})^T (\loss(y_T) - \loss(x^*))
    &\leq \loss(x_0) - \loss(x^*) + (\frac{\mu}{2} + 2 \rho) \norm{x_0 - x^*}^2 \leq (\tfrac{\lipschitz{} + \mu}{2} + 2 \rho) \norm{x_0 - x^*}^2,
\end{align*}
with the last inequality following from the $\lipschitz{}$-smoothness of $\loss$. The second statement follows similarly via
\begin{align*}
    (1 + \tfrac{\gamma}{2})^T (\loss(y_T) - \loss(x^*))
    &\leq \loss(x_0) - \loss(x^*) + (\frac{\mu}{2} + 2 \rho) \norm{x_0 - x^*}^2 \leq (2 + \tfrac{4 \rho}{\mu}) (\loss(x_0) - \loss(x^*)).
\end{align*}
We therefore turn to showing  Claim~\ref{claim:potential}.

First, we note that the potential function \eqref{eq:lyapunov} we are proving differs from the usual potential function that is used to prove Nesterov's acceleration, which we will denote by $\Phi_{\mathrm{orig}}$:
\begin{align*}
    \Phi_{\mathrm{orig}}(t) = (1 + \gamma)^t [(\loss(y_t) - \loss(x^*)) + \tfrac \mu 2 \norm{z_t - x^*}^2].
\end{align*}
Fact~\ref{fact:perturbedpotential} says roughly that, for any $t \in [T]$, we can hypothetically recover a large part of the usual proof of Nesterov's acceleration by bounding the difference $\Phi_{\mathrm{orig}}(t) - \Phi_{\mathrm{orig}}(t-1) \leq 0$, if only we could remove the perturbation at timestep $t$, i.e. set $\Delta_{t} = 0$ (but keep in place the perturbations $\Delta_1, \dots, \Delta_{t-1}$ from previous timesteps).
\begin{fact}
\label{fact:perturbedpotential}
Fix any $t \in [T]$ and let $\tilde z_{t} \asseq \tfrac{1}{\tau}(y_{t} + \beta(y_{t} - y_{t-1})) - \tfrac{1-\tau}{\tau}y_{t}$.
Then,
\begin{equation}
    (1+\gamma)(\loss(y_{t})-\loss(x^*)) - (\loss(y_{t-1})-\loss(x^*)) + \tfrac{\mu}{2}\left((1+\gamma)\|\tilde{z}_{t}-x^*\|^2-\|z_{t-1}-x^*\|^2\right) \leq 0 \label{eq:potential-drop}
\end{equation}
\end{fact}

Next, in Fact~\ref{fact:perturbedapprox}, we show that, since we can argue that $\Phi_{\mathrm{orig}}(t) \leq \Phi_{\mathrm{orig}}(t-1)$ if we ignore the perturbation $\Delta_{t}$, we can argue that $\Phi(t) \leq \Phi(t-1)$ even taking the perturbation $\Delta_{t}$ into account.
That is, Fact~\ref{fact:perturbedapprox} shows that the left-hand side of \eqref{eq:potential-drop} upper bounds $\Phi(t) - \Phi(t-1)$, so long as our stated assumption on the perturbation norm $\norm{\Delta_{t}}$ holds.

\begin{fact}
\label{fact:perturbedapprox}
For any $t \in [T]$, the left-hand side of \eqref{eq:potential-drop} upper bounds $\Phi(t) - \Phi(t-1)$; equivalently,
\begin{align}
    &(1 + \tfrac{\gamma}{2})(\tfrac{\mu}{2} \norm{z_{t} - x^*}^2  + 2 \rho \norm{y_{t} - x^*}^2) - 2 \rho \norm{y_{t-1} - x^*}^2 \nonumber \\
 &   \leq \tfrac \gamma 2 (\loss(y_{t}) - \loss(x^*)) +  (1 + \gamma) \tfrac{\mu}{2} \norm{\tilde z_{t} - x^*}^2. \label{eq:perturbedapprox}
\end{align}
\end{fact}

Plugging Fact~\ref{fact:perturbedapprox}'s \eqref{eq:perturbedapprox} into Fact~\ref{fact:perturbedpotential}'s \eqref{eq:potential-drop}, we recover our main claim that $\Phi(t) - \Phi(t-1) \leq 0$.
We conclude by turning to prove Facts~\ref{fact:perturbedpotential} and \ref{fact:perturbedapprox}.

\begin{proof}[Proof of Fact~\ref{fact:perturbedpotential}]
The proof of this fact closely follows \citet{gupta2017potential}'s proof of Nesterov's accelerated gradient method in smooth strongly convex settings.

We begin by upper bounding the first summand in the left-hand side of \eqref{eq:potential-drop}.
Since $\loss$ is smooth and $y_{t}$ is a gradient step on $\loss$ from $x_{t-1}$, we have that $\loss(y_{t}) \leq \loss(x_{t-1}) - \frac 1 {2\lipschitz{}} \norm{\nabla x_{t-1}}^2$ and thus
\begin{align}
\label{eq:b1}
 &(1 + \gamma) (\loss(y_{t}) - \loss(x^*)) - (\loss(y_{t-1}) - \loss(x^*)) \nonumber \\
 &
\leq \loss(x_{t-1}) - \loss(y_{t-1}) + \gamma (\loss(x_{t-1}) - \loss(x^*)) - \frac{1+\gamma}{2\lipschitz{}} \norm{\nabla \loss(x_{t-1})}^2.
\end{align}
Using the $\mu$-strong convexity of $\loss$, we can further bound part of the right-hand side by
\begin{align}
\loss(x_{t-1}) - \loss(y_{t-1}) + \gamma (\loss(x_{t-1}) - \loss(x^*))
\leq \; & \inner{\nabla \loss(x_{t-1})}{x_{t-1} - y_{t-1}}\nonumber \\
& \;+ \gamma \para{\inner{\nabla \loss(x_{t-1})}{x_{t-1} - x^*}  - \frac \mu 2 \norm{x_{t-1} - x^*}^2}.
\label{eq:b2}
\end{align}

Plugging \eqref{eq:b2} and the identity $z_{t-1} = \frac{1}{\tau}x_{t-1} - \frac{1-\tau}{\tau}y_{t-1}$ into \eqref{eq:b1}, direct algebra yields an upper bound on the first summand of \eqref{eq:potential-drop}:
\begin{align}
& (1 + \gamma) (\loss(y_{t}) - \loss(x^*)) - (\loss(y_{t-1}) - \loss(x^*)) \nonumber \\
& \leq \frac{1}{1 + \gamma} \inner{\nabla \loss(x_{t-1})}{\gamma (z_{t-1} - x^*) + \gamma^2 (x_{t-1} - x^*)} - \frac{\mu \gamma}{2} \norm{x_{t-1} - x^*}^2 - \frac{1 + \gamma}{2 \lipschitz{}} \norm{\nabla \loss(x_{t-1})}^2. 
\label{eq:origpotentialleft}
\end{align}

Next, we turn to upper bounding the second summand 
in \eqref{eq:potential-drop}.
Plugging the iterate definitions $y_{t} = x_{t-1} - \frac{1}{\lipschitz{}} \nabla \loss(x_{t-1})$ and ${z}_{t-1} = \frac{1}{\tau} {x}_{t-1} - \frac{1-\tau}{\tau} {y}_{t-1}$ into our definition of $\tilde{z}_{t} = \frac{1}{\tau}\para{(1 - \tau)y_{t} - (1 - 2 \tau)y_{t-1}}$, we can recover the identity
$$\tilde{z}_{t} = \frac{1}{1+\gamma} z_{t-1} + \frac{\gamma}{1+\gamma} x_{t-1} - \frac{\gamma}{\mu(1 + \gamma)} \nabla \loss(x_{t-1}).$$
Plugging this identity for $\tilde{z}_{t}$ into the expression $\tfrac{\mu}{2}\left((1+\gamma)\|\tilde{z}_{t}-x^*\|^2-\|z_{t-1}-x^*\|^2\right)$, direct algebra yields the identity
\begin{align}
\tfrac{\mu}{2}\left((1+\gamma)\|\tilde{z}_{t}-x^*\|^2-\|z_{t-1}-x^*\|^2\right)
 = \;&  \frac{1 + \gamma}{2 \lipschitz{}} \norm{\nabla \loss(x_{t-1})}^2 \nonumber \\
 & - \frac{1}{1 + \gamma} \inner{\nabla \loss(x_{t-1})}{\gamma (z_{t-1} - x^*) + \gamma^2 (x_{t-1} - x^*)} \nonumber \\
& \; + \frac{\mu \gamma}{2} \norm{x_{t-1} - x^*}^2 - \frac{\mu \gamma}{2(1+\gamma)} \norm{z_{t-1} - x_{t-1}}^2. 
\label{eq:origpotentialright}
\end{align}
Summing \eqref{eq:origpotentialleft} and \eqref{eq:origpotentialright} yields the following upper bound on the left-hand side of \eqref{eq:potential-drop}:
\begin{align*}
    &
    (1+\gamma)(\loss(y_{t})-\loss(x^*)) - (\loss(y_{t-1})-\loss(x^*)) + \tfrac{\mu}{2}\left((1+\gamma)\|\tilde{z}_{t}-x^*\|^2-\|z_{t-1}-x^*\|^2\right) \\
    &\leq - \frac{\mu \gamma}{2(1+\gamma)} \norm{z_{t-1} - x_{t-1}}^2 \leq 0.
\end{align*}
\end{proof}

\begin{proof}[Proof of Fact~\ref{fact:perturbedapprox}]
To show this fact, we can observe from direct algebra that it suffices to prove the following inequalities:
\begin{gather}
    \tfrac \gamma 2 (\loss(y_{t}) - \loss(x^*))
    \geq (2 + \tfrac \gamma 2) 2 \rho \norm{y_{t} - x^*}^2,\label{eq:c4} \\
    (1 + \gamma) \norm{\tilde z_{t} - x^*}^2 \geq (1 + \tfrac \gamma 2)  \norm{z_{t} - x^*}^2 -  \tfrac {4 \rho} \mu (\norm{y_{t-1} - x^*}^2 + \norm{y_{t}-x^*}^2).\label{eq:c5}
\end{gather}

The inequality \eqref{eq:c4} follows directly from the $\mu$-convexity of $\loss$, and the identity $\rho = \frac{\mu \gamma}{4(4+ \gamma)}$, as
\[\tfrac{\gamma}{2}(\loss(y_{t}) - \loss(x^*)) \geq \tfrac{ \mu\gamma}{4}\|y_{t}-x^*\|^2 = (2+\tfrac{\gamma}{2})\cdot 2  \rho\|y_{t}-x^*\|^2.  \]

To show \eqref{eq:c5}, we first recall the general fact that $c (1+\delta)\|v\|^2 \geq c \|v + w\|^2 - c (1+\frac{1}{\delta})\|w\|^2$ for any $\delta, c > 0$.
Choosing $v = \tilde z_{t} - x^*$, $w = z_{t} - \tilde z_{t}$, $c = 1 + \tfrac \gamma 2$, and $\delta = \frac{\gamma}{\gamma + 2} > 0$, this yields
\begin{align}
\label{eq:c3}
1+\gamma\|\tilde{z}_{t}-x^*\|^2 \geq (1 + \tfrac \gamma 2) \|z_{t}-x^*\|^2 - (2 + \tfrac{2}{\gamma})(1 + \tfrac \gamma 2)\|\tfrac{1}{\tau}\Delta_{t}\|^2.
\end{align}
Using the perturbation norm bound, we can therefore bound
\[(2 + \tfrac{2}{\gamma})(1+\tfrac{\gamma}{2})\|\tfrac{1}{\tau}\Delta_{t}\|^2  \leq \tfrac {2\rho} \mu \|y_{t}-y_{t-1}\|^2 \leq   \tfrac {4\rho} \mu(\|y_{t}-x^*\|^2 + \|y_{t-1} - x^*\|^2),\]
where the last inequality uses the triangle inequality; plugged into \eqref{eq:c3}, this yields \eqref{eq:c5} as desired.
\end{proof}
\end{proof}

\fi

%% file: src/experiments.tex
In this section we provide some proof-of-concept experiments to validate our theoretical results. 

\subsection{Deep Linear Networks and Squared Losses}
As our main theoretical results in Section~\ref{sec:deep-linear-theory} apply to the case of deep linear networks, we consider the same function class in our experiments on synthetic data with the square loss. Formally, the output space $\cY = \reals^d$, and for a predictor $x \mapsto W_t x$, we consider the loss
\begin{align}
\ell(W) = \mathbb{E}_{(x,y) \sim D} \left[ \tfrac 1 2\|W x - y\|^2 \right].
\end{align}

We consider a data distribution $D$ where the samples $(x,y)$ are drawn as follows. Let $W^*$ and be the ``ground truth'' positive definite matrix and let $\Sigma$ be the data covariance matrix. We first sample $x \sim N(0,\Sigma)$ and then conditioned on $x$, the output $y$ is generated as $y = W^* x + \xi$, where $\xi$ is a mean zero random variable.
We can then write the expected square loss explicitly as
\begin{align}
\ell(W) &= \mathbb{E}_{(x,y) \sim D}\left[\tfrac{1}{2}\|Wx - y\|^2\right] = \tfrac{1}{2} \Tr((W-W^*)\Sigma(W-W^*)^\top) + \tfrac 1 2\mathbb{E} \left[ \|\xi\|^2 \right].
\label{eq:squaredloss}
\end{align}

Note that for the case of the squared loss described above, the condition number of the expected loss depends on the covariance matrix $\Sigma$, i.e., $\kappa = \frac{\sigma_{\text{max}}(\Sigma)}{\sigma_{\text{min}}(\Sigma)}$.

\paragraph{Stacking updates.} For the specific case of the squared loss we get the following closed form expression of the stacking updates (see Eq. \eqref{eq:linearupdate})
\begin{align}
W_{t+1}
= (W_{t} + \beta (W_{t} - W_{t-1})W_{t-1}^{-1} W_{t})(I - \tfrac 1\lipschitz{} \Sigma) + \tfrac 1\lipschitz{} W^* \Sigma.
\label{eq:linear-sq-loss}
\end{align}
Here $L$ is the smoothness of the loss which depends on the largest singular value of $\Sigma$, $L = \sigma_{\text{max}}(\Sigma)$.

\paragraph{Nesterov's updates.} Similarly, we get the following closed form expression for obtaining Nesterov's updates (see Eq. \eqref{eq:nesterov-update}) for the case of squared loss 
\begin{align}
W_{t+1} 
= (W_{t} + \beta (W_{t} - W_{t-1}))(I - \tfrac 1\lipschitz{} \Sigma) + \tfrac 1\lipschitz{} W^* \Sigma.
\label{eq:nesterov-sq-loss}
\end{align}

For both stacking and Nesterov's updates we set $\beta = \frac{\sqrt{\kappa}-1}{\sqrt{\kappa}+1}$.

\subsection{Synthetic data experiments} 

We compare the performance of the three updates namely vanilla gradient descent, stacking updates (Eq. \ref{eq:linear-sq-loss}) and exact Nesterov updates (Eq. \ref{eq:nesterov-sq-loss}). Here at each stacking stage only the last layer is updated which matches our theoretical setup faithfully. In \if\arxiv0{Appendix~\ref{app:additional-expts}}\else{Section~\ref{sec:experiments_main}}\fi{}, we also consider the effect of training all the layers in each stacking stage which is closer to how stacking is applied in practice.

We consider points in $d=20$ dimensions. We generate the ground truth $W^*$ to be of the form $I + \sigma S$, where $S$ is a random positive semi-definite matrix of spectral norm $1$ and $\sigma$ is a parameter defining the closeness of $W^*$ to identity. For a given $\kappa > 1$, we generate a random covariance matrix $\Sigma$. Finally, we sample the noise $\xi$ in the output from a mean zero Gaussian with a standard deviation of $0.1$.

\begin{figure}
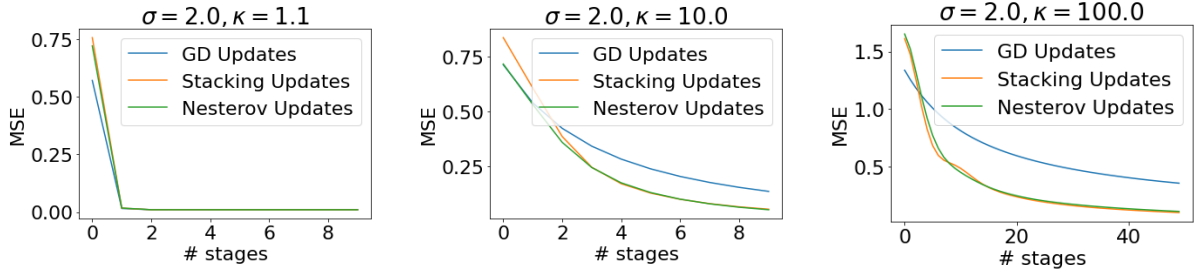

\centering
\begin{minipage}{.3\linewidth}
  \centering
  \includegraphics[width=\linewidth]{figs/stacking_static_sigma_2_kappa_1.pdf}
\end{minipage}
\quad
\begin{minipage}{.3\linewidth}
  \centering
  \includegraphics[width=\linewidth]{figs/stacking_static_sigma_2_kappa_10.pdf}
\end{minipage}
\quad
\begin{minipage}{.3\linewidth}
  \centering
  \includegraphics[width=\linewidth]{figs/stacking_static_sigma_2_kappa_100.pdf}
\end{minipage}
\caption{Mean squared error (MSE) vs. number of stacking stages. We observe that as the data becomes more ill conditioned both the stacking updates and Nesterov's updates demonstrate faster convergence than vanilla gradient descent.}
\label{fig:static_kappa}
\end{figure}
Figure \ref{fig:static_kappa} shows the performance of the three types of updates as the problem becomes more ill conditioned, i.e. as a function of $\kappa$. As expected, at small values of the condition number there is no advantage of the stacking updates over vanilla gradient descent. However, for ill conditioned data the stacking updates converge much faster than gradient descent. We also observe that the convergence of the stacking updates mirrors very closely the convergence behavior of the exact Nesterov's updates. 

\if\arxiv1
To further understand the relationship between the stacking updates and Nesterov's updates, in Figure \ref{fig:static_sigma} we show the performance of the two as the distance of $W^*$ from identity increases. As can be seen from the figures, when $W^*$ is farther from identity the stacking updates behave qualitatively different from Nesterov's updates during the initial phase where the loss for stacking updates explodes before converging in later stages. This suggests that in practice there may be a better way to initialize a stacking stage by making the initialization closer to the ideal Nesterov's updates. While in the case of deep linear networks and the squared loss we have a closed form expression for such an initialization, in general this is a hard problem.

\begin{figure}
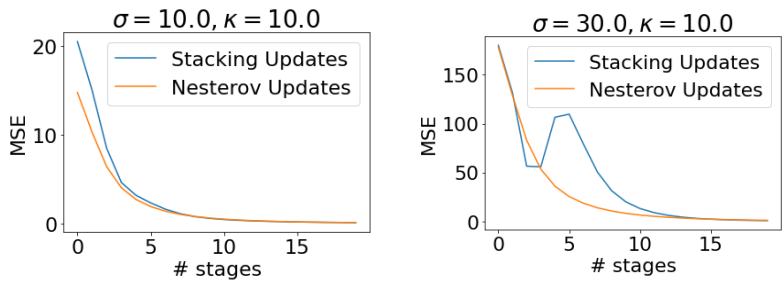

\centering
\begin{minipage}{.3\linewidth}
  \centering
  \includegraphics[width=\linewidth]{figs/stacking_static_sigma_10_kappa_10.pdf}
\end{minipage}
\quad
\begin{minipage}{.3\linewidth}
  \centering
  \includegraphics[width=\linewidth]{figs/stacking_static_sigma_30_kappa_10.pdf}
\end{minipage}
\caption{Mean squared error (MSE) vs. number of stacking stages. The figure compares stacking updates and Nesterov's updates as $W^*$ becomes farther from Identity, i.e. $\sigma$ increases. We observe that for higher values of $\sigma$ the stacking updates display a diverging behavior in the initial stages.}
\label{fig:static_sigma}
\end{figure}

Next we consider the case where in each stacking stage we train all the layers of the deep linear network. We use the same data generation procedure as described above. We perform $10$ stages of stacking where in each stage we perform $2$ steps of gradient descent with a learning rate of $1/L$ where $L$ is the smoothness of the loss function. We train on $1024$ examples with batch size of $32$ and test on $1024$ examples.

We consider two types of stacking based initialization schemes. The first one namely {\em Stacking Init.} initializes the next layer's weight matrix $w^0_{t+1}$ as $\beta w_{t}$. The second scheme namely {\em Nesterov Init.} initializes $w^0_t$ such that we recover the precises Nesterov's updates at initialization, i.e., Eq.~\ref{eq:nesterov-sq-loss}. From the analysis in Section~\ref{sec:functional} the initialization that achieves this amounts to setting $w^0_{t+1}$ as $\beta w_{t} (I + w_t)^{-1}$. 

Figure \ref{fig:sigma} shows the performance of the two stacking initialization schemes as compared to the {\em random} baseline where we initialize the next layer's weight matrix to be a random one.
We again observe that both the stacking schemes outperform the baseline particularly when the data is ill conditioned. 

\begin{figure}
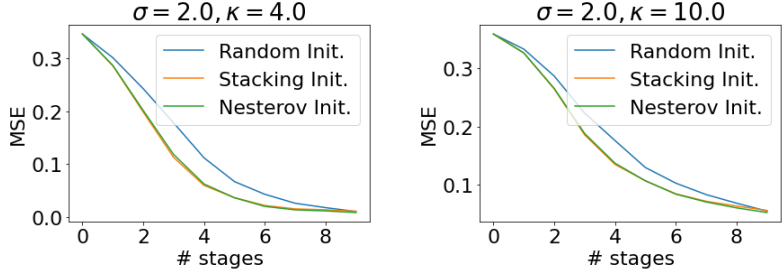

\centering
\begin{minipage}{.3\linewidth}
  \centering
  \includegraphics[width=\linewidth]{figs/stacking_sigma_2_kappa_4_1.pdf}
\end{minipage}
\quad
\begin{minipage}{.3\linewidth}
  \centering
  \includegraphics[width=\linewidth]{figs/stacking_sigma_2_kappa_10_1.pdf}
\end{minipage}
\caption{Mean squared error (MSE) vs. number of stacking stages when training all the layers.}
\label{fig:sigma}
\end{figure}
\fi

\subsection{Stacking for BERT Base with $\beta$ parameters}

\if\arxiv0
\begin{wrapfigure}{r}{0.5\textwidth}
    \centering
    \includegraphics[width=0.35\textwidth]{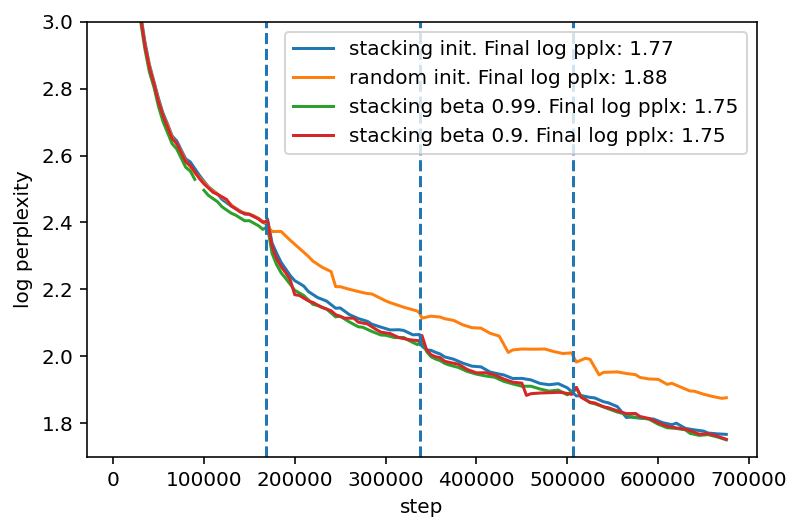}
    \caption{Stacking initialization with trainable $\beta$ parameter multiplying the output of the newly added transformer block. Experimental runs with $\beta$ initialized to $0.99$ and $0.9$ are provided.}
    \label{fig:bert-base-beta}
\end{wrapfigure}
\else
\begin{figure}[!t]
    \centering
    \includegraphics[width=0.65\textwidth]{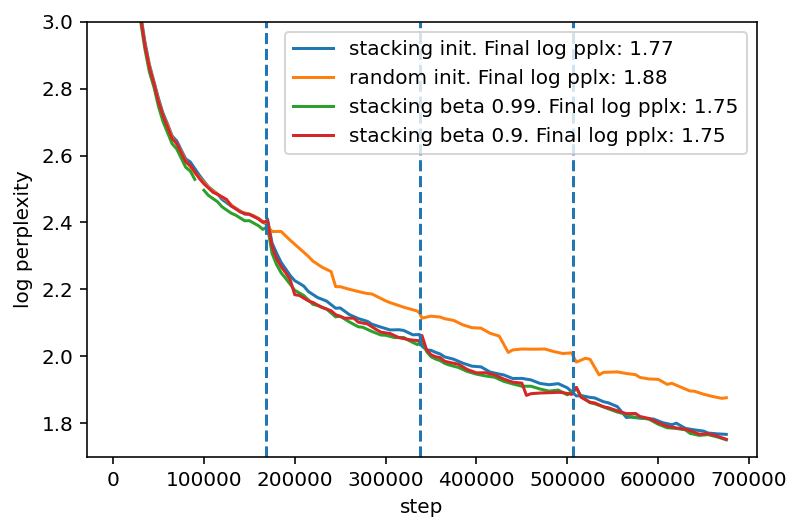}
    \caption{Stacking initialization with trainable $\beta$ parameter multiplying the output of the newly added transformer block. Experimental runs with $\beta$ initialized to $0.99$ and $0.9$ are provided.}
    \label{fig:bert-base-beta}
\end{figure}
\fi

The theory developed in Section~\ref{sec:functional} requires the initialization at the $(t+1)$-th stage to be $f_{t+1}^0 = \beta f_t$ for some $\beta \in [0, 1)$. The introduction of $\beta$ is crucial to get the accelerated convergence rate in Nesterov's method, but the standard stacking initialization doesn't use a $\beta$ parameter. We performed sanity check experiments on BERT Base to ensure that the introduction of the $\beta$ parameter doesn't affect the efficacy of stacking. We introduced a trainable parameter, $\beta$, that multiplies the output of the newly added transformer block in stacking, which is initialized to the values $0.9$ and $0.99$, which are standard settings for momentum parameters. Figure~\ref{fig:bert-base-beta} shows that introduction of the $\beta$ parameter doesn't hurt the efficacy of stacking. The plot also shows that the final log perplexity improves a bit when using trainable $\beta$.

%% file: arxiv.bbl
\newcommand{\etalchar}[1]{$^{#1}$}